\documentclass[10pt,conference]{IEEEtran}
\IEEEoverridecommandlockouts

\usepackage{cite}
\usepackage{amsmath,amssymb,amsfonts}
\usepackage{algorithmic}
\usepackage{graphicx}
\usepackage{textcomp}
\usepackage{xcolor}
\def\BibTeX{{\rm B\kern-.05em{\sc i\kern-.025em b}\kern-.08em
    T\kern-.1667em\lower.7ex\hbox{E}\kern-.125emX}}

\usepackage[colorlinks=true, allcolors=blue]{hyperref}
\usepackage[nolist,nohyperlinks]{acronym}
\usepackage{svg}
\usepackage{tikz}
\usetikzlibrary{calc}
\usepackage{subcaption}
\usepackage{multirow}
\usepackage{orcidlink}

\usepackage{bm}
\usepackage{amsthm} 
\usepackage{xfrac} 
\usepackage{booktabs} 
\usepackage{float} 

\usepackage{amsmath,amsfonts,bm}

\newtheorem{theorem}{Theorem}[section]

\newtheorem{lemma}[theorem]{Lemma}
\newtheorem{property}[theorem]{Property}

\newcommand{\qf}{\mathfrak{q}}
\newcommand{\blf}{\mathfrak{b}}

\def\1{\bm{1}}

\DeclareMathAlphabet{\mathsfit}{\encodingdefault}{\sfdefault}{m}{sl}
\SetMathAlphabet{\mathsfit}{bold}{\encodingdefault}{\sfdefault}{bx}{n}

\def\sC{{\mathbb{C}}}

\def\sN{{\mathbb{N}}}

\newcommand{\E}{\mathbb{E}}
\newcommand{\C}{\mathbb{C}}
\newcommand{\F}{\mathbb{F}}

\newcommand{\R}{\mathbb{R}}
\newcommand{\Cl}{\mathbb{C}l}

\begin{document}

\begin{acronym}
\acro{kan}[KAN]{Kolmogorov-Arnold Network}
\acro{kat}[KAT]{Kolmogorov-Arnold Representation Theorem}
\acro{cvkan}[CVKAN]{Complex-Valued Kolmogorov-Arnold Network}
\acro{cliffkan}[$\Cl$KAN]{Clifford Kolmogorov-Arnold Network}
\acro{sobolcliffkan}[Sobol-$\Cl$KAN]{Sobol Clifford Kolmogorov-Arnold Network}
\acro{cvnn}[CVNN]{Complex-Valued Neural Network}
\acro{mlp}[MLP]{Multilayer Perceptron}
\acro{rbf}[RBF]{Radial Basis Function}
\acro{silu}[SiLU]{Sigmoid Linear Unit}
\acro{mae}[MAE]{Mean Absolute Error}
\acro{mse}[MSE]{Mean Square Error}
\acro{ce}[CE]{Cross-Entropy}
\acro{rqmc}[RQMC]{Randomized Quasi-Monte Carlo}
\acro{qmc}[QMC]{Quasi-Monte Carlo}
\acro{mc}[MC]{Monte Carlo}
\acro{cga}[CGA]{Conformal Geometric Algebra}
\acro{pga}[PGA]{Projective Geometric Algebra}
\acro{ga}[GA]{Geometric Algebra}
\acro{ca}[CA]{Clifford Algebra}
\acro{pde}[PDE]{Partial Differential Equation}
\acro{ega}[EGA]{Euclidean Geometric Algebra}
\end{acronym}

\title{Clifford Kolmogorov-Arnold Networks
}

\makeatletter
\newcommand{\newlineauthors}{
  \end{@IEEEauthorhalign}\hfill\mbox{}\par
  \mbox{}\hfill\begin{@IEEEauthorhalign}
}
\makeatother

\author{\IEEEauthorblockN{Matthias Wolff\textsuperscript{*} \orcidlink{0009-0006-1432-6265}{}}
\IEEEauthorblockA{
\textit{Dept. of Computer Science}\\
\textit{University of Münster}\\
Münster, Germany 
}
\and
\IEEEauthorblockN{Francesco Alesiani\textsuperscript{*} \orcidlink{0000-0003-4413-7247}}
\IEEEauthorblockA{
\textit{Data Science Div.}\\
\textit{NEC Labs Europe}\\
Heidelberg, Germany \\
}
\and
\IEEEauthorblockN{Christof Duhme \orcidlink{0009-0004-1853-2862}}
\IEEEauthorblockA{
\textit{Dept. of Computer Science}\\
\textit{University of Münster}\\
Münster, Germany \\
}
\and
\IEEEauthorblockN{Xiaoyi Jiang \orcidlink{0000-0001-7678-9528}}
\IEEEauthorblockA{
\textit{Dept. of Computer Science}\\
\textit{University of Münster}\\
Münster, Germany \\
}
\thanks{\textsuperscript{*}Equal contribution.}
}
\maketitle

\begin{abstract}
We introduce Clifford Kolmogorov-Arnold Network (ClKAN), a flexible and efficient architecture for function approximation in arbitrary Clifford Algebra spaces. We propose the use of Randomized Quasi-Monte Carlo grid generation as a solution to the exponential scaling associated with higher-dimensional algebras. Our ClKAN also introduces new batch normalization strategies to deal with variable domain input.
ClKAN finds application in scientific discovery and engineering, and is validated in synthetic and physics-inspired tasks. 
\end{abstract}

\textbf{\textcopyright 2026 IEEE}. Personal use of this material is permitted.  Permission from IEEE must be obtained for all other uses, in any current or future media, including reprinting/republishing this material for advertising or promotional purposes, creating new collective works, for resale or redistribution to servers or lists, or reuse of any copyrighted component of this work in other works.
The final published version (M. Wolff, F. Alesiani, C. Duhme and X. Jiang, "Clifford Kolmogorov-Arnold Networks," 2026 International Joint Conference on Neural Networks (IJCNN), Maastricht, Netherlands, 2026, pp. 1-8, doi: 10.1109/IJCNN69854.2026.11728617) is available at IEEE Xplore: \url{https://ieeexplore.ieee.org/document/11728617}

\section{Introduction}
The \ac{kat} \cite{kolmogorov1961representation}
provides an alternative way of representing functions of multiple variables, including continuous and discontinuous functions \cite{ismayilova2024kolmogorov}. While it is a useful theoretical tool, its application was only marginally explored \cite{sprecher2002space}.
After Liu et al.~\cite{liu2025} have recently sparked new interest in the field proposing \acp{kan}, 
there has been a surge of extensions as well as applications of this new architecture \cite{somvanshi2025survey}. 
Especially for function-fitting tasks, \acp{kan} have shown to be a promising alternative to conventional \acp{mlp} \cite{yu2024kan} while also offering intrinsic interpretability.

Wolff et al. \cite{wolff2025cvkan} have transferred the advantages of \acp{kan} into the field of \acp{cvnn} and introduced a \ac{cvkan} by making use of \acp{rbf}. Che et al. \cite{che2026improved} have then adapted the underlying \ac{kat} to the complex domain with mathematical proofs and further improved the \ac{cvkan} by changing the residual activation function and making the shape of the \acp{rbf} learnable.

However, scientific discovery, engineering, computer vision or robotic applications work with high-dimensional representations, for example, to describe electromagnetic fields, weather state variables, three-dimensional objects, or multi-joint robot arms. 
In these scenarios, complex numbers are not sufficient. 

In this work we therefore propose \ac{cliffkan}, an extension to \ac{cvkan} for higher dimensions than complex numbers. 
By increasing the dimensionality of the grid on which the \acp{rbf} are defined 
and using the geometric product defined in the \ac{ca},  we can learn a function on each edge connecting two nodes in the \ac{cliffkan} that maps from $\Cl \mapsto \Cl$, similar to \ac{kan} and \ac{cvkan} that map from $\R \mapsto \R$ or $\sC \mapsto \sC$, respectively.

Overall our contributions can be summarized as:
\begin{itemize}
    \item We extend the CVKAN framework \cite{wolff2025cvkan} to the hyper-complex domain using Clifford Algebras (Section~\ref{sec:method}).
    \item We propose two different types of \acp{rbf} as basis functions (Section~\ref{sec:rbf}).
    \item We propose to transition from a uniform grid to a randomly sampled grid to mitigate the curse of dimensionality. In particular, we introduce the use of \ac{rqmc}
    \emph{scrambled Sobol sequence} for the grid generation (Section~\ref{sec:grid_sobol}). 
    \item We show the expressivity of the proposed approach when using the \emph{Sobol grid} (Section~\ref{sec:expressivity}).
    \item We study the effect of different batch normalization techniques suitable for Clifford domain and \ac{kan} architecture in general (Section~\ref{sec:batchnorm}).
    \item We provide our code of the model and the experiments as an open-source library to promote open science\footnote{Link to code base: \url{https://github.com/M-Wolff/CliffordKAN}}.
    \end{itemize}
In addition, we evaluate (Section \ref{sec:experiments}) \ac{cliffkan} against the conventional \ac{cvkan} (described in Section~\ref{sec:KAN}) in complex-valued experiments, we study the influence of different batch normalization strategies, and the use of the proposed \acp{rbf} (Section~\ref{sec:exp_funcfit}). 
Further, we study the suitability of our proposed \ac{cliffkan} with higher-dimensional \acp{ca} (introduced in Section~\ref{sec:clifford}), e.g. quaternions, in synthetic function-fitting tasks. We show that our proposed \emph{Sobol grid} approach is helpful for reducing the number of parameters (Section~\ref{sec:exp-high-dim}), while providing flexible performances.

\section{Related Work}

\textbf{Kolmogorov-Arnold Function Representation}
theorem provides an alternative tool for function approximation \cite{lai2021kolmogorov}, with the potential to address the curse of dimensionality \cite{poggio2022deep}.
The original KAN architecture has been extended to Convolutional Neural Networks (CNNs)  \cite{ferdausKANICEKolmogorovArnoldNetworks2024,bodner2024convolutional}, and to transformer models \cite{yangKolmogorovArnoldTransformer2024}. To apply \acp{kan} to PolSAR classification, \cite{kuang2025exploring} developed a complex-valued convolutional \ac{kan}.
Furthermore, matrix group equivariance has been integrated into the KAN architecture \cite{huEKANEquivariantKolmogorovArnold2024a}, while \cite{alesianiGeometricKolmogorovArnoldSuperposition2025} provides an extension for geometrical symmetries and invariance.
Adaptive KAN architecture~\cite{alesiani2025variational} reduces the need for hyperparameter optimization. Liu et al. \cite{liu2025kolmogorov} explore the use of KANs in science with symbolic regression. While \cite{wolff2025cvkan,che2026improved} explore the applied and theoretical aspects of complex KAN, and \cite{yu2024residual} studies the role of residual connections. 

\textbf{Clifford Algebra}, or 
Clifford's \ac{ga} provides a powerful mathematical tool for describing spatial relationships and transformations. GA generalizes complex numbers and quaternions into a unified system \cite{lundholm2009clifford}.
\ac{ca} has found application in modeling \acp{pde} \cite{brandstetter2022clifford} or in modeling rotational and translation ($E(3)$) or Lorentz ($O(1,3)$) symmetries 
\cite{brehmer2023geometric,ruhe2023clifford,ruhe2023geometric}, 
which are used to describe interactions in physical systems. 
\ac{cga} is used in robotic vision \cite{bayro2006conformal}, for inverse kinematics \cite{hildenbrand2008inverse}, or in computer vision \cite{dorst2002geometric}.

\section{Background}

\subsection{Clifford Algebra}
\label{sec:clifford}
Clifford Algebra \cite{ruhe2023clifford,lundholm2009clifford,crumeyrolle2013orthogonal} extends Euclidean vector space to a full algebra by introducing the concept of geometric product. Clifford Algebra finds application in various areas of mathematics, physics, science, and engineering \cite{hitzer2024Survey}. If $V$ is a finite-dimensional vector space of dimension $n=\dim V$, over a field $\F$ ($\R$ in the paper) equipped with a quadratic form $\qf: V \to \F$, then $\Cl(V,\qf)$ is the Clifford Algebra, a unitary, associative, noncommutative algebra generated by $V$, where $v^2 = \qf(v)$. 
Each element of the algebra is a linear combination of a finite product of vectors $v_{ij} \in V$, 
\begin{equation*}
    x = \sum_{i \in I} c_i v_{i,1} \dots v_{i,k_i},
\end{equation*}
with $c_i \in \F, k_i \in N_0$, 
$I$ finite.

The bilinear form associated with $\Cl(V,\qf)$ is defined as $\blf(v_1,v_2)=\frac1{2} (\qf (v_1+v_2)-\qf(v_1) -\qf(v_2) )$, which satisfies the identity $v_1v_2+v_2v_1=2 \blf(v_1,v_2), \forall v_1,v_2 \in V$. 
This defines the 
\emph{geometric product} $v_1v_2$ on which the algebra is defined. If $v_1,v_2$ are orthogonal then $\blf(v_1,v_2)=0$ and therefore $v_2 v_1=-v_1 v_2$. 
If $e_1,\dots,e_n$ is an 
orthogonal basis of $V$, then the set of tuples $e_I=e_{i_1} \dots e_{i_{|I|}}, I=(i_1,\dots,i_{|I|}) \subseteq [n]$ 
forms an orthogonal basis for $\Cl(V,\qf)$. Therefore, the dimension of the Clifford Algebra is $D=\dim \Cl(V,\qf) = 2^n$. 
We can partition the algebra into vector spaces $\Cl^{(m)}(V,\qf), m=0,\dots,n$ called \emph{grades} (see \figurename~\ref{fig:clifford}), where the $\dim \Cl^{(m)}(V,\qf) = \binom nm$. Notably, $\Cl^{(0)}(V,\qf)=\F$, and $\Cl^{(1)}(V,\qf)=V$.  The elements of $\Cl^{(2)}(V,\qf)$ and $\Cl^{(3)}(V,\qf)$ are called bivectors and trivectors, while $\Cl^{(n)}(V,\qf)$ contains the pseudoscalar. 
The bilinear operator implies the scalar product $x \cdot y = \blf(x,y) = \frac1{2} (xy+yx), \quad x,y \in \Cl(V,\qf)$, similarly we can define the external product as $ x \wedge y = \frac1{2} (xy-yx)$. We can now write the geometric product between two vectors $v,u \in V$ as
\begin{equation*}
vu = v \cdot u + v \wedge u
\end{equation*}
The resulting object has two components, a scalar $v \cdot u$ and a bivector $v \wedge u$, which represents an area. 
When \ac{ca} is used, e.g. in describing Maxwell's equations \cite{chappell2014geometric}, the electric and magnetic fields can be represented as a single element in the algebra, with $F = E + e_{123}cB$, with $E$ and $B$ the electric and magnetic vectors, while $e_{123}$ is the pseudo-scalar and $c$ is the light speed. Further notice that the element $e_{123}cB$ is a bivector. We can now write, from the original $12$ equations, the Maxwell's equations as $(1/c \nabla_t + \nabla_x) F = 1/\epsilon_0 (\rho - J/c)$, 
where $\rho,J$ represent charges (scalar) and the electric currents (vector), and $\epsilon_0$ the dielectric constant.

For convenience, we refer to the Clifford Algebra as $\Cl_{p,q,r}$, or simply $\Cl$, where the triplet $_{p,q,r}$ counts the number of elements that square to $1,-1$ and $0$, 
i.e. $\qf(e_i) = \{1,-1,0\}$. 

\begin{figure}[t]
    \centering    \includegraphics[trim={0cm 3cm  0cm 1cm},clip,width=\linewidth]{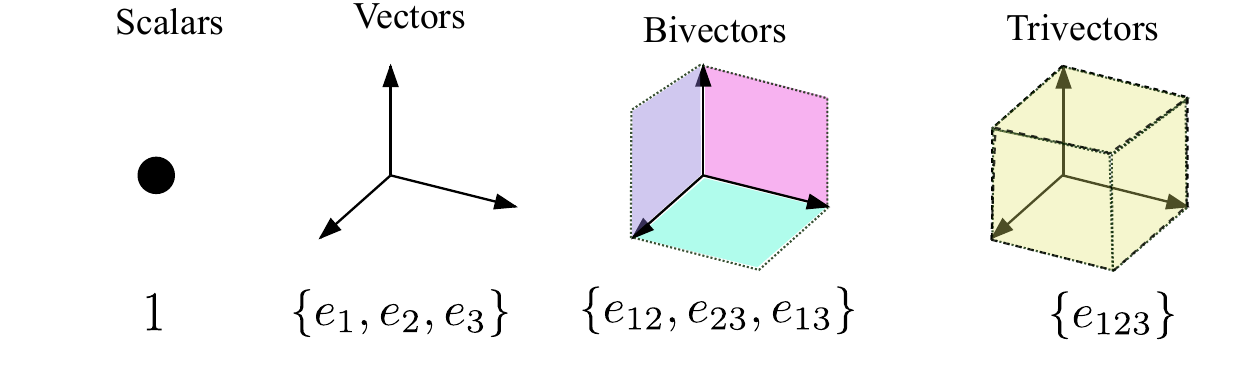}
    \caption{Visualization of $\Cl_{3,0,0}$ grades (scalars, vectors, bivectors and trivectors).
    }
    \label{fig:clifford}
\end{figure}

Notable examples of Clifford Algebras are the \ac{ega}, for example $\Cl_{2,0,0}$ is used in \ac{pde} to model weather \cite{brandstetter2022clifford}, or $\Cl_{3,0,0}$ is used to model Maxwell's equations \cite{chappell2014geometric}. The Projective Geometric Algebra (PGA) $\Cl_{3,0,1}$ and the Conformal Geometric Algebra (CGA) $\Cl_{4,1,0}$ are used for computer vision 
\cite{bayro2006conformal}, and robotic manipulation~\cite{hildenbrand2008inverse} applications. The Spacetime Algebra $\Cl_{3,1,0}$ describes the Minkowski space and is used to model relativistic physics \cite{hestenes2015space}.

\subsection{Complex-valued Kolmogorov-Arnold Network}
\label{sec:KAN}
Wolff et al. \cite{wolff2025cvkan} extended the KAN \cite{liu2025} framework to the complex domain with the \ac{cvkan} inspired by \acp{cvnn}, which are known for improved generalization in signal processing and physics-related tasks.
\ac{cvkan} introduces the complex \ac{rbf} as the base function of the expansion of the univariate function for the \ac{kan} architecture. The complex-valued \ac{rbf} $\phi_\C: \C \rightarrow \R$ is defined as $\phi_\C(x) = \exp (-|x|^2)$, with $|.|$ the absolute value of the complex number \mbox{$x\in\C$}, and centered on a regular grid in $\C$, i.e. grid points $g \in \C$, resulting in the univariate function of the form \mbox{$\Phi_\C(x): \C \to \R = \sum_{g \in G} w_g \phi_\C(x-g)$}. \ac{cvkan} also introduces $\C$\ac{silu}, complex nonlinearity, defined as $\C\text{\ac{silu}} = w \left(\text{\ac{silu}}(\Re(x)) + i \text{\ac{silu}}(\Im(x))\right) + \beta$, where $w,\beta \in \C$ are learnable weight and bias, and $\Re(x)$ and $\Im(x)$ are the real and imaginary part of $x \in \C$, and $\text{\ac{silu}}(x) = x (1+e^{-x})^{-1}$.

\section{Methodology}
\label{sec:method}
In this section, we describe how a \ac{cliffkan} is constructed for any Clifford Algebra $\Cl_{p,q,r}$. The choice of \mbox{$n=p+q+r \in \sN_+$} defines specific algebras (Section~\ref{sec:clifford}).
This algebra is then used for all \acp{rbf} inside the whole \ac{cliffkan} and dictates the rules for all other calculations, e.g. multiplication with weights.

\subsection{\texorpdfstring{\ac{cliffkan}}{ClKAN} Radial Basis Functions}
\label{sec:rbf}
Similar to CVKAN \cite{wolff2025cvkan} we define a \emph{naive \ac{rbf}} \eqref{eq:naiveRBF} with \mbox{$x \in \Cl$} and $\|x\|$ representing the Clifford Algebra $\Cl$ norm of $x$, $\| x\|^2=  x \star \tilde{ x} $, with $\tilde { x}$ the reverse of $x$ and $\star$ the scalar product \cite{dorst2010geometric}.

This formulation for our naive \ac{rbf} corresponds to a mapping of \mbox{$\phi(x): \Cl \rightarrow \R = \Cl^{(0)}$}. 
By using this mapping, the high-dimensional Clifford Algebra space is represented by a real-valued number, the output of the \ac{rbf}, therefore potentially losing information about the rich structure of the input space.

We thus propose the alternative \emph{Clifford \ac{rbf}} \eqref{eq:cliffordspaceRBF} that retains information about the spatial properties of the Clifford-valued input. We achieve this by multiplying the output of the \ac{rbf} with the input itself. 
The product of an element in $\Cl$ by a real number is still an element of $\Cl$, and represents the scaled version of its components. 

\begin{subequations}
\label{eq:kernel_functions}
\begin{minipage}{.45\linewidth}
\begin{align}
    \phi(x) = \exp\left(-\|x\|^2\right)
    \label{eq:naiveRBF}
\end{align}
\end{minipage}
\begin{minipage}{.45\linewidth}
\begin{align}
    \phi(x)_{\Cl} = x ~ \phi(x)
    \label{eq:cliffordspaceRBF}%
\end{align}
\end{minipage}
\vspace{8pt}
\end{subequations}

Equation~\eqref{eq:cliffordspaceRBF} then constitutes a mapping \mbox{$\phi(x)_{\Cl}: \Cl \rightarrow \Cl$} and the output is equal to the input $x \in \Cl$ weighted by the real-valued output of the naive \ac{rbf} $\phi(x)$. The two functions are related by the gradient operator, $\phi(x)_{\Cl} = -\frac1{2} \nabla_x \phi(x)$. Even if $\phi(0)_{\Cl}$ is zero, this does not affect its interpolation capability. 
Following \ac{cvkan}, $\Cl$\ac{silu} is the residual activation function with learnable weight and bias $w,\beta \in \Cl$ for $x \in \Cl$:
\begin{equation}
    \Cl\text{SiLU}(x) = w \text{SiLU}_{\Cl}(x) + \beta
    \label{eq:clsilu}
\end{equation}
where $\text{SiLU}_{\Cl}$ denotes the component-wise application of the SiLU function on each element of $x \in \Cl$ independently.

\subsection{\texorpdfstring{\ac{cliffkan}}{ClKAN} Uniform Grid}
\label{sec:grid}
To construct a learnable activation function \mbox{$\Cl \rightarrow \Cl$}, we need to combine multiple \acp{rbf}.
The \acp{rbf} have to be centered around grid points, so we need to define a grid where each grid point $g$ is also an element of the same underlying Clifford Algebra $g \in \Cl$.

In \ac{cvkan} for complex numbers this grid was defined as a uniform grid with $N_g \times N_g$ grid points $g \in \sC$ within some fixed ranges, e.g. $[-2-2i, 2+2i]$, to likely fall inside the grid of the next layer after batch normalization changes the distribution to mean $\mu=0$ and standard deviation $\sigma=1$. This also allows to use a fixed grid range and not rely on grid extension \cite{liu2025,moody2025automatic} or learnable grid offsets \cite{zheng2025free} that might slow down or complicate the learning process.

This approach can be intuitively transferred to \ac{cliffkan}. Let $D$ be the dimensionality of Clifford Algebra $\Cl$. For example $D=2$ for complex numbers and $D=4$ for quaternions. Then the set of all grid points $G$ consists of $(N_g)^D$ grid points \mbox{g $\in \Cl$}. If we adopt $N_g=8$ grid points per dimension from \ac{cvkan}, then for quaternions we end up with a grid of size $8^4 = 8\times 8\times 8\times 8$ and a total of $4096$ grid points $g \in \Cl_{0,2,0}$.

Consistent with \ac{cvkan} each grid point $g \in G$ is associated with a \ac{rbf} centered around it and thus a learnable weight $w_g \in \Cl$ that controls the influence of the \ac{rbf} at grid point $g$ on the learnable activation function. The \ac{rbf} to use can either be the naive \eqref{eq:naiveRBF} or the Clifford \ac{rbf} \eqref{eq:cliffordspaceRBF} to construct the learnable activation function \eqref{eq:cliffordGridNaiveRBF} or \eqref{eq:cliffordGridCliffordRBF}, respectively.

\begin{subequations}
\label{eq:kan_functions}
\noindent
\begin{minipage}{.48\linewidth}
\begin{align}
   \Phi(x) = \sum_{g \in G} w_g \, \phi(x-g)
   \label{eq:cliffordGridNaiveRBF}
\end{align}
\end{minipage}%
\hfill
\begin{minipage}{.48\linewidth}
\begin{align}
   \Phi_{\Cl}(x) = \sum_{g \in G} w_g \, \phi_{\Cl}(x-g)
   \label{eq:cliffordGridCliffordRBF}
\end{align}
\end{minipage}
\end{subequations}

If we use a uniform grid, also referred to as \emph{full grid}, the number of learnable weights increases exponentially with the number of dimensions in the Clifford Algebra (Section~\ref{sec:clifford}). 
Therefore, we propose an alternative grid generation approach that mitigates the curse of dimensionality.
\subsection{\texorpdfstring{\acl{rqmc}}{RQMC} Sobol Grid}
\label{sec:grid_sobol}
To reduce the number of trainable parameters, 
we introduce a random grid approach. For each dimension of $g \in \Cl_{p,q,r}$ we sample a random number inside a fixed grid range, e.g. $[-2, +2]$, and repeat this process for all of the grid points. Similar to how random search was shown to be a better approach than grid search for hyperparameter optimization~\cite{bergstra2012}, we aim to achieve a sufficient coverage of the hypercube of all possible grid points $[-2, +2]^D$ while requiring fewer grid points than the full grid approach.
The sampled grid points act like supports for our learnable activation function $\Phi(x)$ or $\Phi_{\Cl}(x)$. 
A natural requirement for these points is to cover the hypercube evenly, for each realization, and not only in expectation. 
However, we noticed that naively sampled random grid  points do not necessarily show this property.
We therefore introduce the \ac{rqmc} scrambled Sobol sequence grid, or \emph{Sobol grid}, where we generate the grid points using a quasi-random Sobol sequence~\cite{sobol1967distribution,owen2021strong}.
As shown in \figurename~\ref{fig:gridcomparison}, these quasi-random numbers cover the space more evenly than true random numbers and therefore are better suited for our sampling of grid points.
We call this variant of the model the \ac{sobolcliffkan}. In our experiments, 
we only use the \emph{Sobol random grid}.

\begin{figure}
    \centering
    \includegraphics[width=\linewidth]{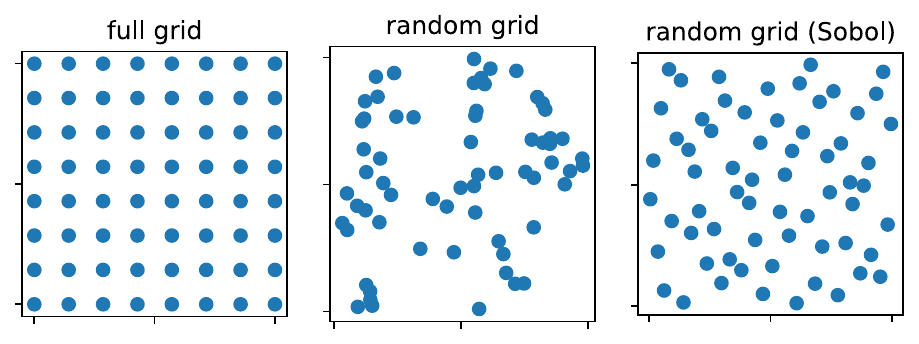}
    \caption{Comparison of full grid, random grid, and quasi-random \emph{Sobol grid} using Sobol sequences with 64 grid points each inside a grid from $(-2, -2)$ to $(2,2)$.}
    \label{fig:gridcomparison}
\end{figure}

\subsection{Expressivity of 
\texorpdfstring{\ac{sobolcliffkan}}{Sobol-CliffordKAN}
}
\label{sec:expressivity}
Using the uniform grid does not scale in high-dimensional space, however, using a uniform random grid results in very high variance. By introducing the \emph{Sobol grid}, we reduce the variance of the model \eqref{eq:cliffordGridCliffordRBF}.
We now analyze the properties of the proposed approach. 
A $(t, m, d)$-net in base $b$, with \mbox{$0 \le t \le m$}, is a set of samples $y_1,\dots,y_{b^m} \in [0,1]^d$ such that for each hyper-cube of volume $b^{t-m}$ it contains exactly $b^t$ points. 
Here, the total number of points is called $n=b^m=|G|$, while in the rest of the paper $n$ is the dimension of the vector space.
We now consider the integral $h(x) = \int_{[0,1]^d} g(y) k(x-y) dy$. If we sample $\{ y_i \}_{i \in [n]}$ in $[0,1]^d$, we can use \ac{mc} to estimate $h(x) \approx \frac1{n} \sum_{i=1}^n g(y_i)k(x-y_i) = \hat{h}_n(x)$. 
We notice that \eqref{eq:cliffordGridNaiveRBF} 
$\Phi(x) = \sum_{g  \in G}  w_g~\phi(x-g) = \hat{h}_n(x)$, with $g(y_i) = w_g$ and $k(x-y_i) = \phi(x-g)$.
If we denote $\sigma^2(h(x)) = \E [ \| h(x)-\mu(h(x)) \|_{\infty}^2], \mu(h(x)) = \int_{[0,1]^d} h(x) dx$, it follows that:
\begin{property}
\label{th:variance}
If the grid points are generated using a scrambled $(t, m, d)$-net in base $b$ 
\cite{owen2021strong,owen1998scrambling}, 
and $g(y),k(x)$ are smooth functions, s.t. $f_x(y)=g(y)k(x-y) \in L^{1+\epsilon}[0,1]^d, \epsilon>0$ (space of $(1+\epsilon)$ integrable functions), and $\exists ~ \Gamma < \infty$, then 
\begin{equation}
\E [ \| \hat{\mu}_n(h(x)) - \mu(h(x)) \|_{\infty}^2  ] \le \Gamma \frac{\sigma^2(h(x))}{n} = O(n^{-1}).
\end{equation}
\end{property}
\begin{IEEEproof}
We apply the result of 
\cite{owen1995randomly,owen2021strong} 
to the integration of function $f_x(y) = g(y) k(x-y)$.
$\E [ \| \hat{\mu}^x_n - \mu_x \|^2  ] \le \Gamma \frac{\sigma_x^2}{n} $, with $\sigma^2_x = \sigma^2(f_x) = \E[\|f_x(y)-\mu_x\|^2]$, $\forall x \in [0,1]^d$, and $\hat{\mu}^x_n = \frac1{n} \sum_{i=1}^n f_x(y_i)$. This is  also true for the maximum of the variances, and therefore the result follows. 
\end{IEEEproof}

We can now state our lemma: 
\begin{lemma}
If we use the \emph{Sobol grid} of size $n$, then 
$ \Phi(x) = \sum_{g  \in G}  w_g~\phi(x-g)$ is an unbiased estimator of  
$h(x) = \int_{[0,1]^d} g(y) \phi(x-y) dy$ 
with variance $O(n^{-1})$.
\end{lemma}
\begin{IEEEproof}
Using the Sobol scrambled sequence \cite{owen1998scrambling}, we define the estimator of $h(x)$ with the condition in Pr.~\ref{th:variance} as $ \hat{h}_n(x) = \frac1{n} \sum_{i=1}^n g(y_i)k(x-y_i)$. As before, we rewrite
$\Phi(x) = \sum_{g  \in G}  w_g~\phi(x-g) = \hat{h}_n(x)$, with $g(y_i) = w_g$ and $k(x-y_i) = \phi(x-g)$. Therefore, from Pr.~\ref{th:variance}, we have that $\Phi(x)$ is an unbiased estimator $h(x)$ with variance $O(n^{-1})$.
\end{IEEEproof}
Therefore, training our KAN network means training the function $g(y)$, evaluated at the grid points. Our proposed \ac{sobolcliffkan} can learn any function that can be written as a convolution with the kernel $\phi(x)$ with an error of
$O(n^{-1})$. We further notice that with this interpretation, the role of the kernel is to smooth the learned function $g(y)$, and for the limit of $\beta \to \infty$ and $n \to \infty$, we recover the original function, if we consider the new kernel $k(x) = \sfrac{\beta}{\sqrt{\pi}} \phi(\beta x) \to_{\beta \to \infty} = \delta(x)$.
Since $h'(x)= \nabla_x h(x) = \int_{[0,1]^d}g(y)\nabla_x\phi(x-y)dy=\int_{[0,1]^d}g(y)\phi'(x-y)dy$, similar result shall extend to \eqref{eq:cliffordGridCliffordRBF}.

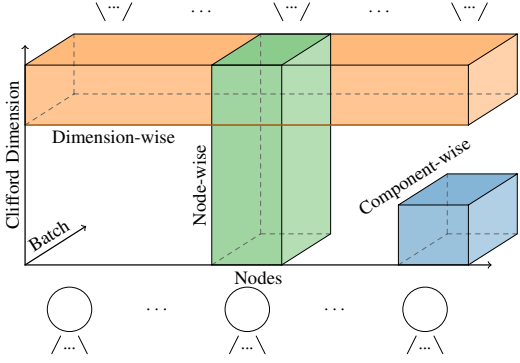
\begin{figure}[t]
    \centering
    \resizebox{0.8\columnwidth}{!}{%
        \DeclareRobustCommand 
\Compactcdots{\mathinner{\cdotp\mkern-4mu\cdotp\mkern-4mu\cdotp}}

\begin{tikzpicture}[line join=round, line cap=round, font=\large]

\definecolor{tabblue}{RGB}{31,119,180}
\definecolor{taborange}{RGB}{255,127,14}
\definecolor{tabgreen}{RGB}{44,160,44}

\def\exx{1.05} \def\exy{0.0}
\def\eyx{0.55} \def\eyy{0.35}
\def\ezx{0.0}  \def\ezy{0.90}

\newcommand{\proj}[4]{%
  \pgfmathsetmacro{\PX}{(#2)*\exx + (#3)*\eyx + (#4)*\ezx}%
  \pgfmathsetmacro{\PY}{(#2)*\exy + (#3)*\eyy + (#4)*\ezy}%
  \coordinate (#1) at (\PX cm,\PY cm);%
}

\newcommand{\makecuboid}[7]{%
  \pgfmathsetmacro{\xA}{#2}%
  \pgfmathsetmacro{\yA}{#3}%
  \pgfmathsetmacro{\zA}{#4}%
  \pgfmathsetmacro{\xB}{#2 + #5}%
  \pgfmathsetmacro{\yD}{#3 + #6}%
  \pgfmathsetmacro{\zE}{#4 + #7}%
  \proj{#1A}{\xA}{\yA}{\zA}%
  \proj{#1B}{\xB}{\yA}{\zA}%
  \proj{#1C}{\xB}{\yD}{\zA}%
  \proj{#1D}{\xA}{\yD}{\zA}

  \proj{#1E}{\xA}{\yA}{\zE}%
  \proj{#1F}{\xB}{\yA}{\zE}%
  \proj{#1G}{\xB}{\yD}{\zE}%
  \proj{#1H}{\xA}{\yD}{\zE}%
}

\newcommand{\drawcuboidfaces}[2]{%
  \fill[#2!25] (#1D)--(#1C)--(#1G)--(#1H)--cycle; 
  \fill[#2!15] (#1A)--(#1B)--(#1C)--(#1D)--cycle; 
  \fill[#2!20] (#1A)--(#1D)--(#1H)--(#1E)--cycle; 
  \fill[#2!35] (#1B)--(#1C)--(#1G)--(#1F)--cycle; 
  \fill[#2!45] (#1A)--(#1B)--(#1F)--(#1E)--cycle; 
  \fill[#2!55] (#1E)--(#1F)--(#1G)--(#1H)--cycle; 
}

\newcommand{\drawcuboidhidden}[1]{%
  \draw[dashed, black!60]
    (#1A)--(#1D)
    (#1D)--(#1C)
    (#1D)--(#1H);
}

\newcommand{\drawcuboidvisible}[1]{%
  \draw[black]
    (#1E)--(#1F)--(#1G)--(#1H)--cycle
    (#1B)--(#1C)--(#1G)
    (#1A)--(#1B)
    (#1A)--(#1E)
    (#1B)--(#1F);
}

\newcommand{\drawaxes}[4]{%
  \proj{#1O}{0}{0}{0}
  \proj{#1X}{#2}{0}{0}
  \proj{#1Y}{0}{#3}{0}
  \proj{#1Z}{0}{0}{#4}

  \draw[->, thick]
    (#1O) -- (#1X)
    node[midway, sloped, below] {Nodes};

  \draw[->, thick]
    (#1O) -- (#1Y)
    node[midway, sloped, above] {Batch};

  \draw[->, thick]
    (#1O) -- (#1Z)
    node[midway, sloped, above] {Clifford Dimension};
}

\makecuboid{D}{0}{0}{3.5}{9.5}{2}{1.5}
\makecuboid{N}{4}{0}{0}{1.5}{2}{5}
\makecuboid{C}{8}{0}{0}{1.5}{2}{1.5}

\drawcuboidfaces{D}{taborange}
\drawcuboidfaces{N}{tabgreen}
\drawcuboidfaces{C}{tabblue}

\drawcuboidhidden{D}
\drawcuboidhidden{N}
\drawcuboidhidden{C}

\drawcuboidvisible{D}
\drawcuboidvisible{N}
\drawcuboidvisible{C}

\draw[draw=none, fill=taborange] (DA) -- (DB) node[pos=0.2, sloped, below] {Dimension-wise};
\draw[draw=none] (NA) -- (NE) node[pos=0.4, sloped, above] {Node-wise};
\draw[draw=none] (CE) -- (CH) node[midway, sloped, above] {Component-wise};

\drawaxes{A}{10}{2.5}{5.5}

\begin{scope}[yshift=-1cm] 
  \draw (1,0) circle (.5);
  \draw (0.6,-1) -- (0.8,-0.6);
  \draw (1.4,-1) -- (1.2,-0.6);
  \node at (1,-0.9) {$\Compactcdots$};
  \node at (3,0) {$\cdots$};
  \draw (5,0) circle (.5);
  \draw (4.6,-1) -- (4.8,-0.6);
  \draw (5.4,-1) -- (5.2,-0.6);
  \node at (5,-0.9) {$\Compactcdots$};
  \node at (7,0) {$\cdots$};
  \draw (9,0) circle (.5);
  \draw (8.6,-1) -- (8.8,-0.6);
  \draw (9.4,-1) -- (9.2,-0.6);
  \node at (9,-0.9) {$\Compactcdots$};
\end{scope}
\begin{scope}[yshift=+5.7cm, yscale=-1, xshift=+1cm]
  \draw (0.6,-.2) -- (0.8,.2);
  \draw (1.4,-.2) -- (1.2,.2);
  \node at (1,-.1) {$\Compactcdots$};
  \node at (3,0) {$\cdots$};
  \draw (4.6,-.2) -- (4.8,.2);
  \draw (5.4,-.2) -- (5.2,.2);
  \node at (5,-.1) {$\Compactcdots$};
  \node at (7,0) {$\cdots$};
  \draw (8.6,-.2) -- (8.8,.2);
  \draw (9.4,-.2) -- (9.2,.2);
  \node at (9,-.1) {$\Compactcdots$};
\end{scope}

\end{tikzpicture}
    }
    \caption{Visualization of \emph{dimension-wise} (orange), \emph{node-wise} (green) and \emph{component-wise} (blue) batch normalization within one layer of the \ac{kan} architecture.}    
    \label{fig:batchnorms}
\end{figure}

\subsection{Batch Normalization}
\label{sec:batchnorm}
Since the batch normalization approaches of \ac{cvkan} only consider complex numbers \cite{wolff2025cvkan}, we introduce three types of normalization for the arbitrarily high-dimensional Clifford Algebras (see \figurename~\ref{fig:batchnorms}), to make sure inputs to the next layer of our \ac{cliffkan} likely stay within the grid ranges:

\begin{itemize}
\item \emph{Dimension-wise} batch normalization applies batch normalization over each individual dimension of all nodes in one layer combined.
\item In \emph{node-wise} batch normalization we normalize the input in each single node of the \ac{kan} architecture over all its dimensions at once. This approach is intuitively the most suitable approach for \acp{kan}, as each node can compute different correlations in the data and therefore each node's output should be normalized independently of other nodes in the same layer.
\item \emph{Component-wise} batch normalization is a combination of \emph{dimension-wise} and \emph{node-wise} batch normalization, as each node and each dimension get normalized independently over all samples in the batch.
\end{itemize}

Additionally \emph{no-normalization} can be applied, but then there is a high chance that inputs into the next layer exceed the grid ranges and are far away from every grid point, so that the \ac{rbf}'s output for such out of range data is $0$.

\begin{figure*}[ht]
    \centering
    \begin{subfigure}{\textwidth}
    \centering
    \includegraphics[width=\textwidth]{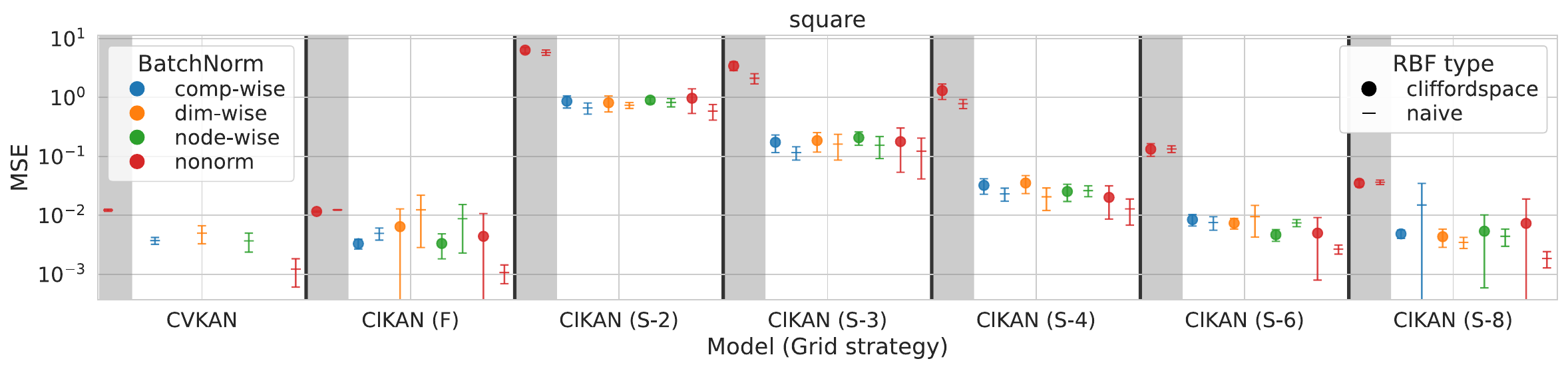}
    \caption{\emph{square} dataset.}
    \label{fig:funcfit_cv_square}
    \end{subfigure}\\
    \begin{subfigure}{\textwidth}
    \centering
    \includegraphics[width=\textwidth]{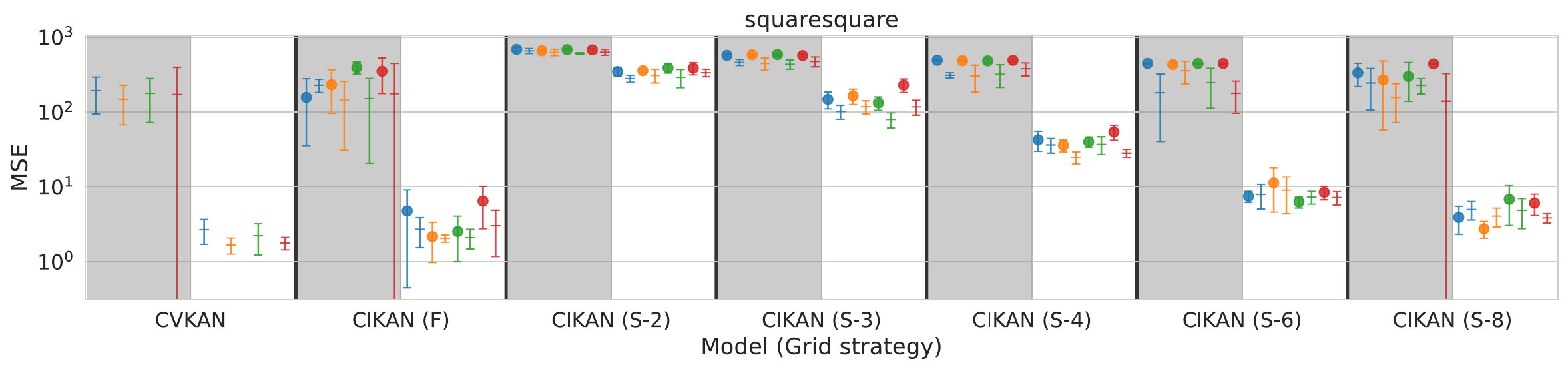}
    \caption{\emph{squaresquare} dataset.}
    \label{fig:funcfit_cv_squaresquare}
    \end{subfigure}   
    \caption{Overview of \ac{mse} for all experiments for complex-valued ($\Cl_{0,1,0}$) synthetic function-fitting tasks
    on datasets \emph{square} and \emph{squaresquare} with color indicating the type of batch normalization, shape the type of \ac{rbf} and shading the architecture size with shaded regions representing the small model architecture. X-axis labels correspond to \ac{cvkan} baseline followed by \ac{cliffkan} \textbf{F}ull grid and \emph{\textbf{S}obol grid} with number of grid points per dimension. For better readability S-5 and S-7 have been omitted. Y-axis shown in log-scale.}
    \label{fig:funcfit_cv}
\end{figure*}

\section{Experiments}
\label{sec:experiments}
In all of our experiments we applied a $5$-fold cross-validation, where we also re-initialized the weights and \emph{Sobol grids} for each run, and used a validation loss plateau scheduler with an initial learning rate of $0.1$ - in contrast to \cite{wolff2025cvkan} who used a learning rate of $0.01$ -, a reduction factor $0.9$, patience of $20$ epochs and a threshold of $0.001$. We also applied early stopping after the validation loss has not decreased by at least $0.001$ in the last $200$ epochs. For the full grid approach we adapted the number of grid points per dimension $D$ from \ac{cvkan} as $8$, so that the total number of grid points is $(N_g)^D=8^D$.
For the \emph{Sobol grid} we selected $2$ to $8$ grid points per dimension to study the minimum required number of random grid points for similar performance as full grid.
We also chose grid ranges of $[-2, 2]$ for each dimension compliant with \ac{cvkan}. For all regression tasks, we trained using \ac{mse} and evaluated using \ac{mse} and \ac{mae} metric, while for classification tasks we used \ac{ce} for training and additionally evaluated the accuracy.

\subsection{Complex-valued Basic Synthetic Formulas}
\label{sec:exp_funcfit}
First, we evaluated whether our \ac{cliffkan} achieves the same performance as \ac{cvkan} on the four simple, synthetic complex-valued function-fitting tasks defined in \cite{wolff2025cvkan}, where the functions are $\{ x^2, \sin(x),  x_1 x_2, (x_1^2+x_2^2)^2 \}$ called \emph{square}, \emph{sin}, \emph{mult} and \emph{squaresquare}.
We use $5\,000$ samples for train- and validation-split combined and another $5\,000$ samples for test-split and adapt the same architecture sizes as \cite{wolff2025cvkan}.

We also evaluated on larger datasets more closely related to the real-world. We used the same complex-valued \emph{holography} dataset from \ac{cvkan} with $100\,000$ samples for train and validation split combined and another $100\,000$ samples for the test split.

Additionally we briefly analyzed the performance on the \emph{knot} dataset \cite{davies2021advancing}, which was also used for evaluation in \ac{cvkan}.

\subsection{Higher-dimensional Clifford Algebras}
\label{sec:exp-high-dim}
To analyze how our \ac{cliffkan} performs in higher dimensions than complex, we used the same \emph{square}, \emph{mult} and \emph{squaresquare} formulas to generate datasets with higher dimensionality.
We evaluate the performance on key Geometrical Algebras: 1) Euclidean GA over $\R^2$, i.e. $\Cl_{2,0,0}$, 2) the quaternion isomorphic $\Cl_{0,2,0}$, 3) the Conformal GA $\Cl_{1,1,0}$, and 4) the one-dimensional \ac{pga} $\Cl_{1,0,1}$.
Because the underlying Clifford Algebras are four-dimensional in contrast to the previously used two-dimensional complex values, we need more data points to cover the space sufficiently. For complex-valued function-fitting we sampled $5\,000$ data points in a range $[-2,2]$ for each dimension. Therefore, we had a sampling density of $\frac{5000}{4^2}$. In order to achieve the same sampling density in a four-dimensional hypercube, we need $\frac{N}{4^4} = \frac{5000}{4^2}$ and therefore $N = 16 \cdot 5000 =  80\,000$ samples instead for combined train- and validation- as well as test-split.

We ran experiments with the most promising batch normalization and \ac{rbf} computation strategy given by the previous experiments (cf. Section~\ref{sec:exp_funcfit}).
Furthermore we conducted exhaustive experiments for grid type and number of random grid points, to see how well our proposed \emph{Sobol grid} approach behaves in higher dimensions.

\section{Results}
\label{sec:results}
In this section we will show and describe the most interesting results. We omit the plot for complex-valued function-fitting on \emph{sin} and \emph{mult} dataset and leave out \emph{Sobol grids} with $N_g \in \{5,7\}$ for better readability.
We will make all the results available in a machine readable way together with additional plots and tables inside our repository for the code base.
\subsection{Comparison to baseline}
\label{sec:results_baseline}
First we want to compare our \ac{cliffkan} to the baseline \ac{cvkan} as well as to its improved version \cite{che2026improved}. \tableautorefname~\ref{tab:baselineComp} compares the best performing model from \ac{cvkan} by Wolff et al. \cite{wolff2025cvkan}, improved \ac{cvkan} by Che et al. \cite{che2026improved} and our \ac{cliffkan} for each complex-valued function 
fitting dataset.
It should be noted that in contrast to \ac{cvkan} that used $5\,000$ data points for 5-fold cross-validation training and reported validation losses, we introduce an additional test split of the same size and report test losses but still only do training and cross-validation on $5\,000$ data points. This dataset setup was used for \ac{cvkan} with learning rate (lr) of $0.1$ and our \ac{cliffkan}. It can be seen that the increased learning rate lr=$0.1$ makes the results of \ac{cvkan} far better and that our \ac{cliffkan} achieves similar results with the same learning rate. The improved version of \ac{cvkan} is better than standard \ac{cvkan} with $0.01$ learning rate, but cannot keep up with the finer tuned learning rate.

\begin{table}[t]
    \centering
        \caption{Comparison of the Best Performing Model (\ac{mse}) for Each of the Four Complex-valued Function-Fitting Tasks. Values for Che et al. \cite{che2026improved} and Wolff et al. with lr=0.01 \cite{wolff2025cvkan} Are Copied from Their Respective Paper. Values for \ac{cliffkan} Stem from the Best Models Across All \ac{rbf}, Batch Normalization and Grid Strategies and All Evaluated Architecture Sizes.}
    \label{tab:baselineComp}
    \begin{tabular}{|c|c|c|c|c|c|}
        \toprule
        \multirow{2}{*}{Dataset}  & \multicolumn{2}{|c|}{CVKAN} & improved  & \multirow{2}{*}{\ac{cliffkan}}\\
        & lr=0.01 \cite{wolff2025cvkan} & lr = 0.1 & CVKAN \cite{che2026improved}& \\
        \midrule
        square & 0.013 & 0.001 & 0.009 & 0.001\\
        sin & 0.005 & 0.001 & 0.005 & 0.001\\
        mult & 0.045 & 0.005 & 0.029 & 0.002\\
        squaresquare & 8.150 & 1.665 & 7.355 & 2.049 \\
        \bottomrule
    \end{tabular}
\end{table}

\subsection{Complex-valued Basic Synthetic Formulas}
\label{sec:results_cv_basic}
\figurename~\ref{fig:funcfit_cv_square} and \figurename~\ref{fig:funcfit_cv_squaresquare} show that our full grid \ac{cliffkan} performs as well as \ac{cvkan} with the same number of parameters.
The performance of the \emph{Sobol grid} strategy gets better with an increasing number of grid points per dimension and for $N_g=8$ grid points per dimension the \emph{Sobol grid} approach is on par with the full uniform grid. The larger architectures perform consistently better than smaller architectures.
\figurename~\ref{fig:funcfit_cv_square} has smaller shaded regions and shows only the \emph{no-normalization} setting for small architectures because the small architecture size for the \emph{square} dataset is $[1,1]$ and we never apply batch normalization on the final output of the model. Therefore, other batch normalization approaches than \emph{no-normalization} are not applicable for this architecture size.

For the choice of \ac{rbf} calculation the naive \eqref{eq:naiveRBF} and Clifford approach \eqref{eq:cliffordspaceRBF} perform similarly well. When comparing the different batch normalization approaches no clear winner can be spotted. In some cases (e.g. \emph{square}, full grid and naive \ac{rbf} with the big architecture in \figurename \ref{fig:funcfit_cv_square}) \emph{no-normalization} performs best, in other cases (e.g. \emph{squaresquare}, full grid and Clifford \ac{rbf} with the big architecture in \figurename \ref{fig:funcfit_cv_squaresquare}) \emph{node-wise} or \emph{dim-wise} batch normalization are the best choice.

\begin{table}[hb!]
    \centering
         \caption{Comparison of \ac{cvkan} and \ac{cliffkan} on the \emph{Holography} Dataset for the Two Largest Architectures and Node-wise Batch Normalization. Choice of \ac{rbf} is Clifford \ac{rbf} \eqref{eq:cliffordspaceRBF} for \ac{cliffkan}. GT=Grid Type $\in \{$\textbf{F}ull Grid and \emph{\textbf{S}obol Grid}$\}$, $N_g$ = Number of Grid Points per Dimension, \mbox{$N_p$~=~Number of Parameters}}
    \label{tab:holography}
\begin{tabular}{|c|c|c|c|c|c|}
\toprule
Model & $N_p$ & GT & $N_g$ & Test MSE & Test MAE \\
\bottomrule
 \multicolumn{6}{c}{$3 {\times} 10 {\times} 3 {\times} 1$} \\
\toprule
CVKAN & 8342& F & 8 & ${\bf 0.021}$ \scriptsize{$\pm 0.003$} & ${\bf 0.090}$\scriptsize{$\pm 0.011$} \\
 \hline
   \multirow{8}{*}{\ac{cliffkan}} & 8342 & F & 8 & ${\bf 0.030} $ \scriptsize{$\pm 0.005$} & ${\bf 0.092}$\scriptsize{$\pm 0.008$} \\
  & 782 & S & 2 & $2.799 $ \scriptsize{$\pm 0.780$} & $1.211 $\scriptsize{$\pm 0.167$} \\
  & 1412 & S & 3 & $0.713 $ \scriptsize{$\pm 0.135$} & $0.602 $\scriptsize{$\pm 0.069$} \\
  & 2294 & S & 4 & $0.275 $ \scriptsize{$\pm 0.133$} & $0.348 $\scriptsize{$\pm 0.099$} \\
  & 3428 & S & 5 & $0.101 $ \scriptsize{$\pm 0.030$} & $0.192 $\scriptsize{$\pm 0.034$} \\
  & 4814 & S & 6 & $0.070 $ \scriptsize{$\pm 0.019$} & $0.154 $\scriptsize{$\pm 0.031$} \\
  & 6452 & S & 7 & $0.043 $ \scriptsize{$\pm 0.006$} & $0.117 $\scriptsize{$\pm 0.009$} \\
   & 8342 & S & 8 & $0.036 $ \scriptsize{$\pm 0.008$} & $0.111 $\scriptsize{$\pm 0.012$} \\
\bottomrule
   
\multicolumn{6}{c}{$3 {\times} 10 {\times} 5 {\times} 3 {\times} 1$} \\
\toprule
CVKAN&  12972 & F & 8 & ${\bf 0.016} $ \scriptsize{$\pm 0.001$} & ${\bf 0.066}$\scriptsize{$\pm 0.003$} \\
  \hline
 \multirow{8}{*}{\ac{cliffkan} } &  12972 & F & 8 & ${\bf 0.025} $ \scriptsize{$\pm 0.004$} & ${\bf 0.069}$\scriptsize{$\pm 0.006$} \\
   &  1212 & S & 2 & $1.891 $ \scriptsize{$\pm 0.275$} & $0.952 $\scriptsize{$\pm 0.078$} \\
  &  2192 & S & 3 & $0.272 $ \scriptsize{$\pm 0.036$} & $0.358 $\scriptsize{$\pm 0.022$} \\
  &  3564 & S & 4 & $0.093 $ \scriptsize{$\pm 0.021$} & $0.179 $\scriptsize{$\pm 0.017$} \\
 &  5328 & S & 5 & $0.059 $ \scriptsize{$\pm 0.012$} & $0.123 $\scriptsize{$\pm 0.010$} \\
  &  7484 & S & 6 & $0.041 $ \scriptsize{$\pm 0.006$} & $0.100 $\scriptsize{$\pm 0.012$} \\
  &  10032 & S & 7 & $0.029 $ \scriptsize{$\pm 0.004$} & $0.081 $\scriptsize{$\pm 0.005$} \\
  &  12972 & S & 8 & $0.028 $ \scriptsize{$\pm 0.004$} & $0.079 $\scriptsize{$\pm 0.003$} \\
 \bottomrule
 \end{tabular}
\end{table}

\begin{table}[hb]
    \centering
         \caption{Comparison of Batch Normalization Schemes for \ac{cvkan} and \ac{cliffkan} on the \emph{Holography} Dataset for the Largest Architecture $[3, 10, 5, 3, 1]$ and Full Grid with $N_g=8$. Choice of \ac{rbf} is Clifford \ac{rbf} \eqref{eq:cliffordspaceRBF} for \ac{cliffkan}.}
    \label{tab:holography_norms}
\begin{tabular}{|c|c|c|c|c|}
\toprule
Model  & Norm & Test MSE & Test MAE \\
\midrule
\multirow{4}{*}{\ac{cvkan}} &  nodew. & $0.016 $ \scriptsize{$\pm 0.001$} & $0.066 $\scriptsize{$\pm 0.003$} \\
 &  dimw. & $0.015 $ \scriptsize{$\pm 0.003$} & $0.062 $\scriptsize{$\pm 0.008$} \\
 &  compw. & $0.017 $ \scriptsize{$\pm 0.002$} & $0.063 $\scriptsize{$\pm 0.004$} \\
 &  no & $51.059 $ \scriptsize{$\pm 62.526$} & $3.278 $\scriptsize{$\pm 3.951$} \\
\midrule
\multirow{4}{*}{\ac{cliffkan}} &  nodew. & $0.025 $ \scriptsize{$\pm 0.004$} & $0.069 $\scriptsize{$\pm 0.006$} \\
 &  dimw. & $0.021 $ \scriptsize{$\pm 0.003$} & $0.060 $\scriptsize{$\pm 0.005$} \\
 &  compw. & $0.028 $ \scriptsize{$\pm 0.009$} & $0.071 $\scriptsize{$\pm 0.006$} \\
 &  no & $0.867 $ \scriptsize{$\pm 1.061$} & $0.354 $\scriptsize{$\pm 0.382$} \\
\bottomrule
\end{tabular}
\end{table}

\subsection{Complex-valued Larger Datasets}
\label{sec:results_highdims}
The results on the \emph{holography} dataset 
in \tableautorefname~\ref{tab:holography} show that \ac{cliffkan} with full grid and \emph{Sobol grid} with $N_g=8$ and the baseline \ac{cvkan} perform similarly well. For the \emph{Sobol grids} with $N_g=7$, and therefore $\approx75\%$ of parameters, the test \ac{mse} is almost the same compared to a full grid model. With very low number of grid points $N_g \in \{2,3,4\}$ the model cannot approximate the underlying function of the dataset well.

When comparing the different batch normalization approaches in \tableautorefname~\ref{tab:holography_norms}, it can be clearly seen that \emph{no-normalization} performs by far the worst and the model is not able to learn anything. The other three approaches for batch normalization perform on par with very low differences between them.

We also analyzed the performance of our model on the knot dataset (Section \ref{sec:experiments}), but the accuracies of all models regardless of grid strategy and grid size were so close to each other and mostly above $90\%$ that we chose not to include the results as they would not give additional insight to recognize differences between the candidates.

\begin{figure*}[t]
    \centering
    \includegraphics[width=.96\textwidth]{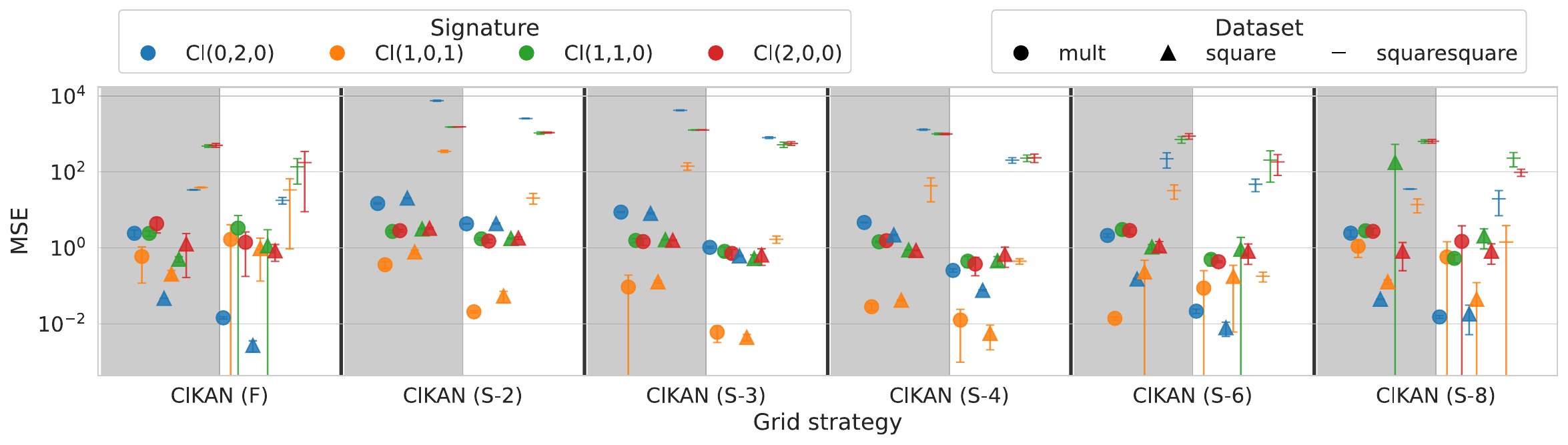}
    \caption{Function-fitting experiments on higher-dimensional Clifford Algebras. Color indicates the Clifford Algebra, shape indicates the dataset \emph{mult}, \emph{square} and \emph{squaresquare} and shading indicates the model architecture per dataset with shaded regions representing the small model architecture. X-axis labels correspond to \textbf{F}ull grid and \emph{\textbf{S}obol grid} with number of grid points per dimension. For better readability S-5 and S-7 have been omitted. Y-axis shown in log scale.}
        \label{fig:funcfit_highdims} 
\end{figure*}

\subsection{Higher-dimensional Clifford Algebras}
\figurename~\ref{fig:funcfit_highdims} shows all experiments on higher-dimensional Clifford Algebras. Similar to Section \ref{sec:results_cv_basic}, \emph{Sobol grid} with the same number of parameters performs approximately as good as the full grid approach. However, in this four-dimensional Clifford Algebra, the \emph{Sobol grid} experiments with $N_g \in \{3,4\}$ show very promising results with a drastically reduced number of parameters ($\approx 2 - 6\%$ of parameters). Overall, the \emph{mult} and \emph{square} datasets can be learned with a lower \ac{mse} than \emph{squaresquare} dataset. The $\Cl_{1,0,1}$ benefits the most from the \emph{Sobol grid} approach and shows even better performance with $N_g=4$ than for the full grid.

In \tableautorefname~\ref{tab:high_dimensional_clifford} it is shown, that for $\Cl_{1,0,1}$ and the \emph{square} and \emph{squaresquare} dataset the results of \emph{Sobol grid} are consistently better than full grid. Some grid sizes, e.g. $N_g\in \{5,6\}$ for \emph{square} and $N_g=8$ for \emph{squaresquare}, perform worse than other grid sizes with high standard deviations.

\begin{table}[bt]
    \centering
    \caption{Comparison of \ac{cliffkan} on the $\Cl_{1,0,1}$ for \emph{square} and \emph{squaresquare} Datasets, Largest Architectures and \emph{node-wise} Batch Normalization. Choice of \ac{rbf} is Clifford \ac{rbf} \eqref{eq:cliffordspaceRBF}. GT=Grid Type $\in \{$\textbf{F}ull Grid and \emph{\textbf{S}obol Grid}$\}$, $N_g$~=~Number of Grid Points per Dimension, \mbox{$N_p$~=~Number of Parameters.}}
    \label{tab:high_dimensional_clifford}
\begin{tabular}{|c|c|c|c|c|}
\toprule
$N_p$ & GT & $N_g$ & Test MSE & Test MAE \\
\bottomrule
\multicolumn{3}{l}{square} & \multicolumn{2}{l}{$1 {\times} 2 {\times} 1$} \\
\toprule
65572 & F & 8 & $0.968 $ \scriptsize{$\pm 0.834$} & $0.585 $\scriptsize{$\pm 0.381$} \\
 292 & S & 2 & $0.054 $ \scriptsize{$\pm 0.018$} & $0.187 $\scriptsize{$\pm 0.032$} \\
 1332 & S & 3 & ${\bf 0.004}$ \scriptsize{$\pm 0.001$} & ${\bf 0.054}$\scriptsize{$\pm 0.005$} \\
 4132 & S & 4 & $0.006 $ \scriptsize{$\pm 0.004$} & $0.058 $\scriptsize{$\pm 0.020$} \\
 10036 & S & 5 & $0.102 $ \scriptsize{$\pm 0.199$} & $0.148 $\scriptsize{$\pm 0.217$} \\
 20772 & S & 6 & $0.177 $ \scriptsize{$\pm 0.171$} & $0.270 $\scriptsize{$\pm 0.196$} \\
 38452 & S & 7 & $0.027 $ \scriptsize{$\pm 0.042$} & $0.096 $\scriptsize{$\pm 0.079$} \\
 65572 & S & 8 & $0.045 $ \scriptsize{$\pm 0.076$} & $0.114 $\scriptsize{$\pm 0.105$} \\
\bottomrule
\multicolumn{3}{l}{squaresquare} & \multicolumn{2}{l}{$2 {\times} 4 {\times} 2 {\times} 1$} \\
\toprule
295068 & F & 8 & $33.586 $ \scriptsize{$\pm 32.646$} & $3.492 $\scriptsize{$\pm 2.202$} \\
 1308 & S & 2 & $20.578 $ \scriptsize{$\pm 6.499$} & $3.194 $\scriptsize{$\pm 0.481$} \\
 5988 & S & 3 & $1.685 $ \scriptsize{$\pm 0.367$} & $0.900 $\scriptsize{$\pm 0.108$} \\
 18588 & S & 4 & $0.445 $ \scriptsize{$\pm 0.074$} & $0.453 $\scriptsize{$\pm 0.035$} \\
 45156 & S & 5 & $0.270 $ \scriptsize{$\pm 0.082$} & $0.345 $\scriptsize{$\pm 0.054$} \\
 93468 & S & 6 & ${\bf 0.178}$ \scriptsize{$\pm 0.052$} & ${\bf 0.285}$\scriptsize{$\pm 0.036$} \\
 173028 & S & 7 & $0.203 $ \scriptsize{$\pm 0.077$} & $0.304 $\scriptsize{$\pm 0.063$} \\
 295068 & S & 8 & $1.425 $ \scriptsize{$\pm 2.414$} & $0.609 $\scriptsize{$\pm 0.557$} \\
\bottomrule
\end{tabular}
\end{table}

\subsection{Discussion}

In our experiments (Section~\ref{sec:results}) we found that the way of calculating the \acp{rbf} does not seem to have a big influence on the model's performance, while the type of batch normalization to use depends on the dataset and there can be cases where \emph{no-normalization} at all seems beneficial, while in other cases models without batch normalization do not train at all. If normalization is helpful, the concrete type of batch normalization - be it \emph{node-wise}, \emph{dimension-wise} or \emph{component-wise} - does not matter much.

The \emph{squaresquare} dataset is by far the most challenging one to train on. 
We suspect this depends on the large variability of its output values, compared to \emph{square}.
We noticed very early in our experiments that \ac{cvkan} as well as \ac{cliffkan} were highly dependent on a good initialization. Therefore, we have increased the learning rate by a factor of $10$ to a value of $0.1$ in comparison to \cite{wolff2025cvkan}. This enables the models to drift further away from the weights given by random initialization and therefore reach more favorable regions. The results across different training runs and during cross-validation also became more stable. This small change alone already allowed us to substantially outperform the improved \ac{cvkan} \cite{che2026improved}.

For higher dimensions (Section~\ref{sec:results_highdims}) we have shown that our proposed \emph{Sobol grid} can be a good tool for parameter reduction in higher dimensions to mitigate the otherwise exponential growth of parameters for the full grid. Some Clifford Algebras, e.g. $\Cl_{1,0,1}$, even benefit from the \emph{Sobol grid} approach in terms of accuracy with drastically reduced parameters. Although it must be noted that the \emph{Sobol grid} sometimes also produces worse results than full grid and some grid sizes in our experiments for specific datasets produce high standard deviations.
This is a general discovery in the experiments, since some experimental setups, that should work in theory, only work for some cross-validation runs in practice therefore producing high \ac{mse} on average and a high standard deviation. With the increased learning rate the experiments overall became more stable and reproducible across runs, but some instabilities can still be observed across all tasks.

\section{Conclusion}
We have shown that our \ac{cliffkan} achieves similar performance on complex-valued datasets compared to \ac{cvkan} as a baseline.
Additionally, we have demonstrated that \ac{cliffkan} can also be successfully applied in higher dimensions for function-fitting tasks.

Furthermore, we have shown that our introduced \emph{Sobol grid} can be used to mitigate the curse of dimensionality for higher-dimensional Clifford Algebras, where a conventional full uniform grid approach would require an exponential number of parameters.
Surprisingly, there are cases where the limited number of parameters in a relatively small \emph{Sobol grid} actually benefit the model's performance and lead to better results than the full grid.

Further research could be dedicated towards learnable \ac{rbf} shape and alternative residual activation functions other than \ac{silu} with improved learning rate, similarly to \cite{che2026improved} for \ac{cvkan}.
Additionally, future work could focus on improving the stability of \acp{kan} and \ac{cvkan} and \ac{cliffkan} in particular, making them even less dependent on good initializations and more robust across different runs.

\section*{Acknowledgment}
LLMs (ChatGPT 5.2 and Google Gemini) were used for sentence-level reformulation, or as a search tool, but not for text generation.

The work was supported by the Deutsche Forschungsgemeinschaft DFG (SPP2363).

\bibliographystyle{IEEEtran}
\bibliography{References,clifford,sobol,kan}

@inproceedings{owen1995randomly,
  title={Randomly permuted (t, m, s)-nets and (t, s)-sequences},
  author={Owen, Art B},
  booktitle={Monte Carlo and Quasi-Monte Carlo Methods in Scientific Computing},
  pages={299--317},
  year={1995},
  organization={Springer}
}

@article{owen2021strong,
  title={A strong law of large numbers for scrambled net integration},
  author={Owen, Art B and Rudolf, Daniel},
  journal={SIAM Review},
  volume={63},
  number={2},
  pages={360--372},
  year={2021},
  publisher={SIAM}
}

@article{lundholm2009clifford,
  title={Clifford algebra, geometric algebra, and applications},
  author={Lundholm, Douglas and Svensson, Lars},
  journal={arXiv:0907.5356},
  year={2009}
}

@book{crumeyrolle2013orthogonal,
  title={{Orthogonal and symplectic Clifford algebras: Spinor structures}},
  author={Crumeyrolle, Albert},
  volume={57},
  year={2013},
  publisher={Springer Science \& Business Media}
}

@article{ruhe2023clifford,
  title={Clifford group equivariant neural networks},
  author={Ruhe, David and Brandstetter, Johannes and Forr{\'e}, Patrick},
  journal={NeurIPS},
  pages={62922--62990},
  year={2023}
}

@INPROCEEDINGS{wolff2025cvkan,
  author={Wolff, Matthias and Eilers, Florian and Jiang, Xiaoyi},
  booktitle={IJCNN}, 
  title={{CVKAN}: Complex-Valued Kolmogorov-Arnold Networks}, 
  year={2025},
  volume={},
  number={},
  pages={1-9},
    doi={10.1109/IJCNN64981.2025.11227425}
}

@InProceedings{che2026improved,
author="Che, Rui
and af Klinteberg, Ludvig
and Aryapoor, Masood",
title="Improved Complex-Valued {Kolmogorov--Arnold} Networks with Theoretical Support",
booktitle="24th EPIA Conference on Artificial Intelligence",
year="2026",
publisher="Springer Nature Switzerland",
pages="439--451",
isbn="978-3-032-05176-9"
}

@inproceedings{liu2025,
 author = {Liu, Ziming and Wang, Yixuan and Vaidya, Sachin and Ruehle, Fabian and Halverson, James and Soljacic, Marin and Hou, Thomas and Tegmark, Max },
 booktitle = {ICLR},
 pages = {70367--70413},
 title = {{KAN}: Kolmogorov\textendash Arnold Networks},
 year = {2025}
}

@article{yu2024kan,
  title={{KAN} or {MLP}: A fairer comparison},
  author={Yu, Runpeng and Yu, Weihao and Wang, Xinchao},
  journal={arXiv:2407.16674},
  year={2024}
}

@article{bergstra2012,
  author  = {James Bergstra and Yoshua Bengio},
  title   = {Random Search for Hyper-Parameter Optimization},
  journal = {Journal of Machine Learning Research},
  year    = {2012},
  volume  = {13},
  number  = {10},
  pages   = {281--305}
}

@article{davies2021advancing,
  title={Advancing mathematics by guiding human intuition with {AI}},
  author={Davies, Alex and Veli{\v{c}}kovi{\'c}, Petar and Buesing, Lars and Blackwell, Sam and Zheng, Daniel and Toma{\v{s}}ev, Nenad and Tanburn, Richard and Battaglia, Peter and Blundell, Charles and Juh{\'a}sz, Andr{\'a}s and others},
  journal={Nature},
  volume={600},
  number={7887},
  pages={70--74},
  year={2021},
  publisher={Nature Publishing Group}
}

@article{somvanshi2025survey,
  title={A survey on kolmogorov-arnold network},
  author={Somvanshi, Shriyank and Javed, Syed Aaqib and Islam, Md Monzurul and Pandit, Diwas and Das, Subasish},
  journal={ACM Computing Surveys},
  volume={58},
  number={2},
  pages={1--35},
  year={2025},
  publisher={ACM New York, NY}
}

@article{liu2025kolmogorov,
  title={Kolmogorov-Arnold networks meet science},
  author={Liu, Ziming and Tegmark, Max and Ma, Pingchuan and Matusik, Wojciech and Wang, Yixuan},
  journal={Phys. Rev. X},
  volume={15},
  pages={041051},
  year={2025},
  publisher={APS}
}

@inproceedings{ruhe2023geometric,
  title={Geometric {Clifford} algebra networks},
  author={Ruhe, David and Gupta, Jayesh K and De Keninck, Steven and Welling, Max and Brandstetter, Johannes},
  booktitle={ICML},
  pages={29306--29337},
  year={2023}
}

@article{bodner2024convolutional,
  title={Convolutional kolmogorov-arnold networks},
  author={Bodner, Alexander Dylan and Tepsich, Antonio Santiago and Spolski, Jack Natan and Pourteau, Santiago},
  journal={arXiv:2406.13155},
  year={2024}
}

@article{yu2024residual,
  title={Residual kolmogorov-arnold network for enhanced deep learning},
  author={Yu, Ray Congrui and Wu, Sherry and Gui, Jiang},
  journal={arXiv:2410.05500},
  year={2024}
}

@inproceedings{kuang2025exploring,
  title={Exploring Complex-Valued Convolutional {Kolmogorov-Arnold} Networks for {PolSaR} Image Classification},
  author={Kuang, Zuzheng and Bi, Haixia and Lv, Zhiyong and Xu, Chen},
  booktitle={IEEE International Geoscience and Remote Sensing Symposium},
  pages={1945--1949},
  year={2025}
}

@article{moody2025automatic,
  title={Automatic Grid Updates for {Kolmogorov-Arnold} Networks using Layer Histograms},
  author={Moody, Jamison and Usevitch, James},
  journal={arXiv:2511.08570},
  year={2025}
}

@article{zheng2025free,
  title={Free-Knots {Kolmogorov-Arnold} Network: On the Analysis of Spline Knots and Advancing Stability},
  author={Zheng, Liangwewi Nathan and Zhang, Wei Emma and Yue, Lin and Xu, Miao and Maennel, Olaf and Chen, Weitong},
  journal={arXiv:2501.09283},
  year={2025}
}

@article{chappell2014geometric,
  title={Geometric algebra for electrical and electronic engineers},
  author={Chappell, James M and Drake, Samuel P and Seidel, Cameron L and Gunn, Lachlan J and Iqbal, Azhar and Allison, Andrew and Abbott, Derek},
  journal={Proceedings of the IEEE},
  volume={102},
  number={9},
  pages={1340--1363},
  year={2014},
  publisher={IEEE}
}

@book{hestenes2015space,
  title={Space-time algebra},
  author={Hestenes, David},
  year={2015},
  publisher={Springer}
}

@inproceedings{brandstetter2022clifford,
  title={Clifford neural layers for pde modeling},
  author={Brandstetter, Johannes and van den Berg, Rianne and Welling, Max and Gupta, Jayesh K},
  booktitle={ICLR},
  year={2023}
}

@article{bayro2006conformal,
  title={Conformal geometric algebra for robotic vision},
  author={Bayro-Corrochano, Eduardo and Reyes-Lozano, Leo and Zamora-Esquivel, Julio},
  journal={Journal of Mathematical Imaging and Vision},
  volume={24},
  pages={55--81},
  year={2006}
}

@inproceedings{brehmer2023geometric,
  title={Geometric algebra transformers},
  author={Brehmer, Johann and De Haan, Pim and Behrends, S{\"o}nke and Cohen, Taco},
  booktitle={NeurIPS},
  year={2023}
}

@book{dorst2010geometric,
  title={Geometric algebra for computer science (revised edition): An object-oriented approach to geometry},
  author={Dorst, Leo and Fontijne, Daniel and Mann, Stephen},
  publisher={Morgan Kaufmann},
  year={2009}
}

@article{dorst2002geometric,
  title={Geometric algebra: a computational framework for geometrical applications},
  author={Dorst, Leo and Mann, Stephen},
  journal={IEEE Computer Graphics and Applications},
  volume={22},
  number={3},
  pages={24--31},
  year={2002}
}

@article{hildenbrand2008inverse,
  title={Inverse kinematics computation in computer graphics and robotics using conformal geometric algebra},
  author={Hildenbrand, Dietmar and Zamora, Julio and Bayro-Corrochano, Eduardo},
  journal={Advances in applied Clifford algebras},
  volume={18},
  pages={699--713},
  year={2008}
}

@article{hitzer2024Survey,
author = {Hitzer, Eckhard and Kamarianakis, Manos and Papagiannakis, George and Vašík, Petr},
title = {Survey of new applications of geometric algebra},
journal = {Mathematical Methods in the Applied Sciences},
volume = {47},
number = {14},
pages = {11368-11384},
doi = {https://doi.org/10.1002/mma.9575},
eprint = {https://onlinelibrary.wiley.com/doi/pdf/10.1002/mma.9575},
year = {2024}
}

@article{sprecher2002space,
  title={Space-filling curves and Kolmogorov superposition-based neural networks},
  author={Sprecher, David A and Draghici, Sorin},
  journal={Neural Networks},
  volume={15},
  number={1},
  pages={57--67},
  year={2002},
  publisher={Elsevier}
}

@book{kolmogorov1961representation,
  title={On the representation of continuous functions of several variables by superpositions of continuous functions of a smaller number of variables},
  author={Kolmogorov, Andrei Nikolaevich},
  year={1961},
  publisher={American Mathematical Society}
}

@article{ismayilova2024kolmogorov,
  title={On the Kolmogorov neural networks},
  author={Ismayilova, Aysu and Ismailov, Vugar E},
  journal={Neural Networks},
  volume={176},
  pages={106333},
  year={2024},
  publisher={Elsevier}
}

@article{lai2021kolmogorov,
  title={The {Kolmogorov} superposition theorem can break the curse of dimensionality when approximating high dimensional functions},
  author={Lai, Ming-Jun and Shen, Zhaiming},
  journal={arXiv:2112.09963},
  year={2021}
}

@article{poggio2022deep,
  title={How deep sparse networks avoid the curse of dimensionality: Efficiently computable functions are compositionally sparse},
  author={Poggio, Tomaso},
  journal={CBMM Memo},
  volume={10},
  pages={2022},
  year={2022}
}

@inproceedings{ferdausKANICEKolmogorovArnoldNetworks2024,
  author       = {Md Meftahul Ferdaus and
                  Mahdi Abdelguerfi and
                  Elias Ioup and
                  David Dobson and
                  Kendall N. Niles and
                  Ken Pathak and
                  Steven Sloan},
  title        = {{KANICE:} Kolmogorov-Arnold Networks with Interactive Convolutional
                  Elements},
  booktitle    = {4th International Conference on {AI-ML} Systems (AIMLSystems)},
  pages        = {13:1--13:10},
  year         = {2024}
}

@inproceedings{yangKolmogorovArnoldTransformer2024,
  author       = {Xingyi Yang and
                  Xinchao Wang},
  title        = {Kolmogorov-Arnold Transformer},
  booktitle    = {ICLR},
  year         = {2025}
}

@inproceedings{huEKANEquivariantKolmogorovArnold2024a,
  title={Incorporating Arbitrary Matrix Group Equivariance into {KANs}},
  author={Hu, Lexiang and Wang, Yisen and Lin, Zhouchen},
  booktitle = {ICML},
  year={2025}
}

@article{alesianiGeometricKolmogorovArnoldSuperposition2025,
  title = {Geometric {{Kolmogorov-Arnold Superposition Theorem}}},
  author = {Alesiani, Francesco and Maruyama, Takashi and Christiansen, Henrik and Zaverkin, Viktor},
  date = {2025-02-23},
  journal = {arXiv:2502.16664},
  eprinttype = {arXiv},
  eprintclass = {cs},
  urldate = {2025-03-05},
  year={2025},
  pubstate = {prepublished},
}

@article{alesiani2025variational,
  title={Variational Kolmogorov-Arnold Network},
  author={Alesiani, Francesco and Christiansen, Henrik and Errica, Federico},
  journal={arXiv:2507.02466},
  year={2025}
}

@article{sobol1967distribution,
  title={The distribution of points in a cube and the accurate evaluation of integrals},
  author={Sobol', Il'ya Meerovich},
  journal={Zhurnal Vychislitel'noi Matematiki i Matematicheskoi Fiziki},
  volume={7},
  number={4},
  pages={784--802},
  year={1967},
  publisher={Russian Academy of Sciences}
}

@article{owen1998scrambling,
  title={Scrambling {Sobol'} and {Niederreiter-Xing} points},
  author={Owen, Art B.},
  journal={Journal of Complexity},
  volume={14},
  number={4},
  pages={466--489},
  year={1998},
  publisher={Elsevier}
}

\end{document}